\documentclass[journal]{IEEEtran}

\usepackage{hyperref}
\usepackage{url}
\usepackage{algorithm}
\usepackage{algpseudocode}
\usepackage{graphicx}
\usepackage{booktabs}
\usepackage{xcolor}
\usepackage{makecell}
\usepackage{subcaption}
\usepackage{enumitem}
\usepackage{amsthm}
\usepackage{natbib}
\usepackage{amsmath, amsfonts}
\usepackage{cuted}
\newtheorem{lemma}{Lemma}
\newtheorem{theorem}{Theorem}

\DeclareMathOperator*{\argmin}{arg\,min}

\begin{document}
%
\title{DFCA: Decentralized Federated Clustering Algorithm}
%
%
%

\author{
    Jonas~Kirch, 
    Sebastian~Becker, 
    Tiago~Koketsu~Rodrigues,~\IEEEmembership{Senior Member,~IEEE,}
    Stefan~Harmeling%
    \thanks{Jonas Kirch is with the Graduate School of Information Science at Tohoku University, Sendai, Japan. Formerly, he was with the Fraunhofer Institute for Software and Systems Engineering, Dortmund, Germany.}%
    \thanks{Sebastian Becker is with the Fraunhofer Institute for Software and Systems Engineering, Dortmund, Germany.}%
    \thanks{Tiago Koketsu Rodrigues is with the Graduate School of Intormation Science at Tohoku University, Sendai, Japan.}%
    \thanks{Stefan Harmeling is with the Department of Computer Science at TU Dortmund University and with the Lamarr Institute for Machine Learning and AI, Dortmund, Germany.}%
}

%
%

\markboth{Preprint. Accepted for Publication in IEEE Internet of Things Journal}%
{Kirch \MakeLowercase{\textit{et al.}}: DFCA: Decentralized Federated Clustering Algorithm}
%



\maketitle

\begin{abstract}
Clustered Federated Learning has emerged as an effective approach for handling heterogeneous data across clients by partitioning them into clusters with similar or identical data distributions. However, most existing methods, including the Iterative Federated Clustering Algorithm (IFCA), rely on a central server to coordinate model updates, typically requiring stable connectivity, synchronous communication rounds, and global aggregation of client models. These assumptions are difficult to satisfy in decentralized and heterogeneous environments, where clients may only have limited, local communication with a small subset of peers. As a result, such methods create a bottleneck and a single point of failure, limiting their applicability in realistic decentralized learning settings. This limitation is particularly severe in Internet of Things settings, where large numbers of resource-constrained devices, intermittent or sparse connectivity, and dynamic participation make reliance on a central server impractical. In this work, we introduce the Decentralized Federated Clustering Algorithm (DFCA), a fully decentralized clustered federated learning algorithm that enables clients to collaboratively train cluster-specific models without central coordination. DFCA uses a sequential running average to aggregate models from neighbors as updates arrive, providing a communication-efficient alternative to batch aggregation while maintaining clustering performance. Our experiments on various datasets demonstrate that DFCA outperforms other decentralized algorithms and performs comparably to centralized IFCA, even under sparse connectivity, highlighting its robustness and practicality for dynamic real-world decentralized networks.
\end{abstract}


%
\IEEEpeerreviewmaketitle

\section{Introduction}
\IEEEPARstart{F}{ederated} Learning (FL) has emerged as a new paradigm that allows for clients to train Machine Learning (ML) models collaboratively without the need to share their raw data. By enabling collaborative training across multiple devices, FL has gained significant attention in research and industry, especially since distributed computing with different devices has become a crucial component of modern technology. The most known FL implementation strategy, FedAvg \citep{FL}, and most other known FL algorithms assume a setting with a central instance that aggregates the local updates of all clients to form a global model, which is then broadcast back to the network. While effective, this orchestration with a central server introduces several limitations, including a single point of failure, communication delays, and bottlenecks that are often connected to more challenging learning settings with Internet of Things (IoT) devices and mobile phones \citep{lalitha2018fully}. In such environments, devices are often highly heterogeneous in data or computation, network connectivity is intermittent, and communication links are unreliable. These characteristics amplify the central server bottleneck, making synchronous aggregation inefficient or even infeasible in practical IoT deployments, thereby motivating decentralized approaches.

To address the limitations of centralized Federated Learning (CFL), recent research has explored decentralized Federated Learning (DFL), where clients communicate with each other without the need for a central instance \citep{lalitha2018fully}. Decentralized strategies often utilize peer-to-peer (P2P) \citep{lalitha2019peer} or gossip-based \citep{hu2019decentralizedfederatedlearningsegmented, GOSSIP} exchange methods to achieve convergence through direct communication between clients. DFL approaches remove the single point of failure, often reduce communication cost and delays, and improve the overall robustness \citep{DFL}.

Concurrently, clustered FL has appeared as a proposed solution to data heterogeneity across clients, another major issue in ML and FL. In most real-world scenarios, the data is not independently identically distributed (non-IID) over all clients, making global aggregation less efficient and suboptimal. Clustered FL methods attempt to cluster clients into groups with similar data distributions, allowing clusters to capture local patterns and characteristics during training \citep{sattler2019}. The most popular among clustered FL techniques is the Iterative Federated Clustering Algorithm (IFCA) \citep{IFCA}, which is a centralized, training loss-based clustering method, where clusters of clients are evaluated locally after each global training round. As most other clustered FL techniques that have been developed over the recent years also presume central instance coordination, they are not optimized for decentralized learning settings. In this paper, we propose the Decentralized Federated Clustering Algorithm (DFCA) to address this issue.

\textbf{Our contributions:}
\begin{enumerate}
    \item We formulate DFCA, a fully decentralized federated clustering algorithm inspired by IFCA and designed to operate effectively in low-connectivity networks with heterogeneous client data distributions, often seen in IoT deployments.

    \item We incorporate a sequential running-average parameter exchange strategy that preserves clustering performance while enabling communication-efficient updates across the network.
    
    \item Through extensive experiments on various datasets, we demonstrate that DFCA matches the accuracy of the centralized IFCA baseline and outperforms decentralized alternatives. Furthermore, sequential aggregation achieves performance comparable to synchronous batch aggregation, highlighting its practicality for real-world decentralized settings.


\end{enumerate}

After looking at the related work and problem formulation in Sections II and III, we proceed to introduce our method in Section IV, analyze its convergence in Section V, and show our simulation results in Section VI. We will finally conclude this paper's findings in Section VII.

\section{Related Work}

\subsection{Decentralized Federated Learning}
DFL originated from decentralized Stochastic Gradient Descent (SGD) optimization \citep{lian2018asynchronousdecentralizedparallelstochastic} and was later formulated by Lalitha \textit{et al.} \citep{lalitha2018fully} as a distinct concept for FL. During the following years, researchers proposed new frameworks and concepts around DFL, leading to rapid growth of the field. Research aspects of DFL include network topologies \citep{wang_impactofnetworktopology, neglia_roleofnetworktopologyindistributedml, malandrino2021federatedlearningnetworkedge, marfoq2020throughputoptimaltopologydesigncrosssilo, chellapandi_dflmodelupdatetracking}, communication protocols \citep{sun2021decentralizedfederatedaveraging, lalitha2019peer, GOSSIP, kolskova_decentralizedstochasticoptimizationandgossipalgos, hu2019decentralizedfederatedlearningsegmented, bellet2018personalizedprivatepeertopeermachine} and iteration orders \citep{DFL}. Explicit DFL paradigms \citep{chang_distributeddlformedicalimaging, sheller_multiinstitutionaldelmodeling, sheller_multiinstitutional_2, huang2022continuallearningpeertopeerfederated, yuan2023peertopeerfederatedcontinuallearning, assran2019stochasticgradientpushdistributed, roy2019braintorrentpeertopeerenvironmentdecentralized, pappas2021iplsframeworkdecentralized, shi2021overtheairdecentralizedfederatedlearning, chen2022decentralizedwirelessfederatedlearning, wang_matcha} then put these concepts and assumptions in the context of real-world learning settings.  
However, there still remains a gap in performance between CFL and DFL \citep{sun2024which}, especially in low-connectivity settings and in the presence of heterogeneity, which motivated us to have a closer look into decentralized optimization.

\subsection{Clustered FL}
First being introduced by Sattler \textit{et al.} \citep{sattler2019} in 2019, clustered FL addresses the issue of handling heterogeneous data distributions of clients in a network. To optimize performance and adapt to different learning settings, researchers have introduced different methods to cluster the clients into groups with similar data distributions \cite{ClusterFLsurvey}. After the initial introduction of client-side clustered FL algorithms based on client loss minimization  \cite{sattler2019, IFCA}, recent publications have focused on optimizing this strategy in different learning contexts \citep{HYPCLUSTER, li2022, kim2020}. Voting-scheme-based \citep{gong2024} or k-means-based \citep{Long_2022} methods are alternative solutions utilizing client-side clustering. In contrast to the approaches mentioned above, our algorithm works in decentralized, low-connectivity settings without the need for a central instance. Instead of implicitly evaluating neighboring clients, as in Onoszko et al. \citep{onoszko2021decentralizedfederatedlearningdeep}, DFCA explicitly groups clients according to their data distributions. Lin \textit{et al.} \citep{lin2025fedspd} highlighted the potential of decentralized federated clustering methods when they introduced their decentralized soft-clustering algorithm for scenarios in which clients possess multiple data distributions. However, their method, called FedSPD, addresses a soft clustering scenario where each client may hold data from multiple distributions simultaneously, ignoring hard clustering scenarios where each client only has access to one data distribution.  
This motivated us to develop a decentralized approach capable of matching the performance of centralized IFCA in low-connectivity settings with clients holding different data distributions, as described later in Section III.

\section{Preliminaries}

\begin{figure}
    \centering
    \includegraphics[width=0.7\linewidth]{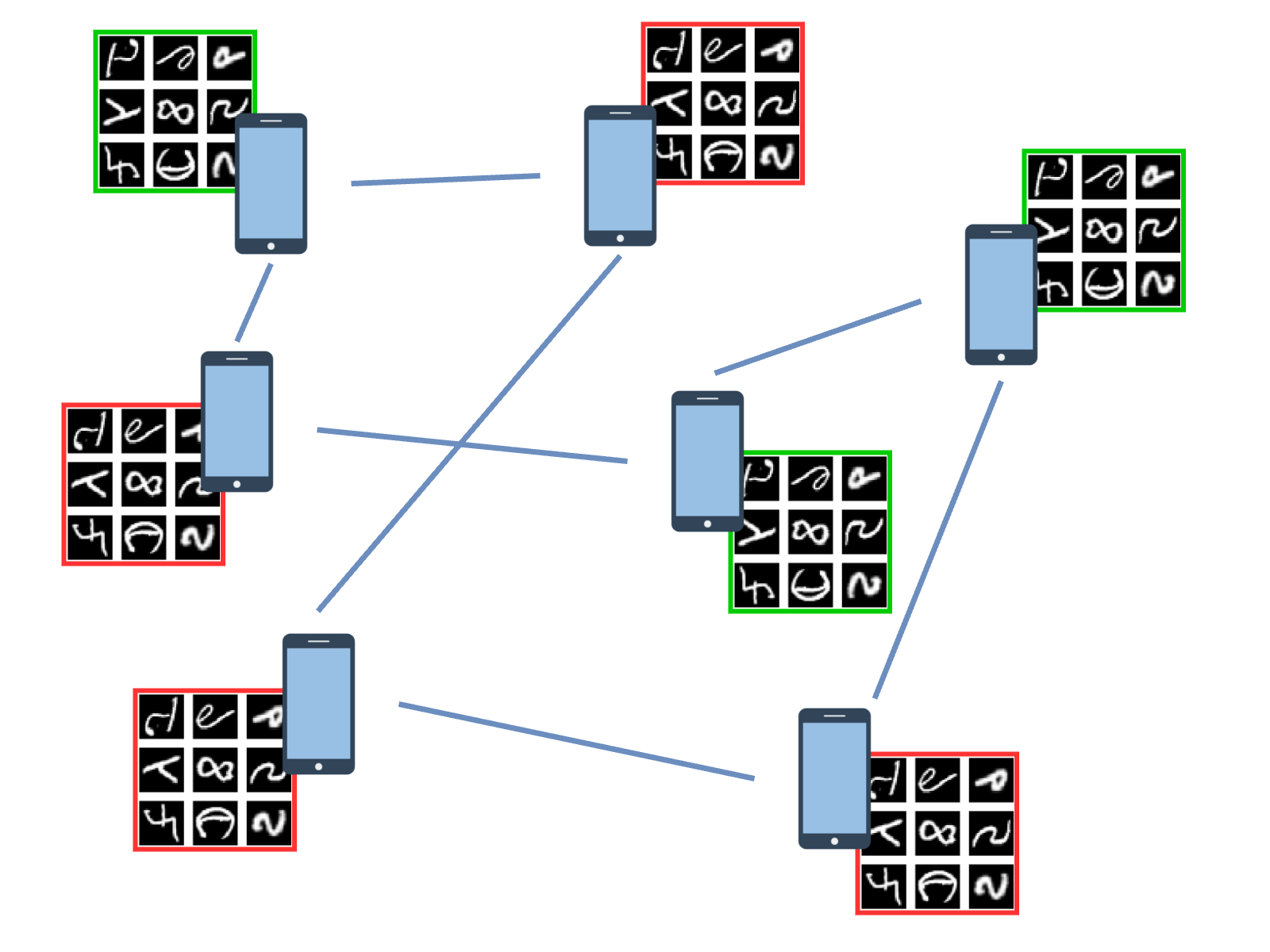}
    \caption{Illustration of the DFCA problem for Rotated EMNIST with two different data distributions}
    \label{fig:problem}
\end{figure}

Let $M$ be a set of $N$ clients that are connected to each other in a graph.  We represent the graph by $N$ sets $\mathcal{N}_i\subset M$, which contain the neighboring clients for each client $i$ (i.e., neighborhood sets). The clients are partitioned into $k$ disjoint clusters $\mathcal{S}_1,...,\mathcal{S}_k\subset M$.  Each cluster is associated with a distinct data distribution $\mathcal{D}_1,...,\mathcal{D}_k$. Our problem setup is illustrated in Figure \ref{fig:problem}, which shows the different data distributions represented by handwritten character digits (EMNIST) rotated by 0, 90, 180, 270 degrees.

For each client $i$, we sample a data set $D_i$ distributed according to $\mathcal{D}_j$ of the associated cluster $j$ meaning that each client has data from one of $k$ data distribution. Additionally, at each client $i$ we store all $k$ Machine Learning models (ML-models), which are parameterized by $\theta_{i,j}$ where $j\in[k]:=\{1,\dots,k\}$ and $k$ is the number of clusters.  Client $i$ will update the parameters $\theta_{i,j}$ of the model, which is associated with its corresponding cluster $j$, by gradient descent using $D_i$.  During aggregation (communication phase) the local models for all clusters are updated using the models of the neighboring clients. Note that the corresponding (assigned) cluster of a client might change after an iteration.

For client-local learning, we consider a loss function $\mathcal{L}(\theta_{i,j}, d)$ that calculates the loss for a single data point $d\in D_i$. These losses can be combined on the client-level and also on the cluster-level:

Let's first consider the loss for an individual client.  We assume that client $i$ is assigned to cluster $j$. Then we write the client-specific objective as 
\begin{equation}
    F_\text{client}(\theta_{i,j},D_i)=\frac{1}{|D_i|}\sum_{d\in D_i} \mathcal{L}(\theta_{i,j},d).
\end{equation} 
Second, we define the loss for each cluster $j$ as the sum of the losses of the associated clients, 
\begin{equation}
    F_\text{cluster}(j)=\sum_{i\in\mathcal{S}_j}F_\text{client}(\theta_{i,j}, D_i).
\end{equation}
Finally, we define the global loss, which combines all data points across all clients into a single number:
\begin{equation}
    F_\text{global}=\sum_{j=1}^k F_\text{cluster}(j)
\end{equation}
Having formulated the loss functions on client- and cluster-level, we next introduce the decentralized learning algorithm DFCA, which allows clients to collaboratively minimize their respective cluster-specific losses while communicating with their neighbors in the graph to exchange results.


\section{Decentralized Federated Clustering Algorithm}

\begin{algorithm}[ht]
\caption{Decentralized Federated Clustering Algorithm (DFCA)}
\label{algo:dfca}
    \begin{algorithmic}[1]
        \State \textbf{Input:} number of clusters $k$, number of iterations $T$
        \State \textbf{Local:} step size $\gamma$, number of local epochs $\tau$
        \State
        \State \underline{DFCA-GI}: initialize $\theta_{i,j}$ per cluster and publish models to all clients
        \State \underline{DFCA-LI}: initialize $\theta_{i,j}$ for all clusters per client (personalized models)
        \State
        \For{$t=0,1,...,T-1$}
            \State $M_t \gets$ subset of worker machines (participating devices)
            \For{worker machine $i \in M_t$}
                \State
                \State \textbf{Step 1:} AssignCluster
                \State $c(i)\leftarrow\underset{j}{\argmin}F_\text{client}(\theta_{i,j}, D_i)$ \\
                \Comment{run local inference on all models}
                \State
                \State \textbf{Step 2:} LocalUpdate
                \For{$q=0,...,\tau-1$}
                \State $\theta_{i, c(i)} \leftarrow  \theta_{i, c(i)} - \gamma \nabla F_\text{client}(\theta_{i,c(i)}, D_i)$ \\
                \Comment{stochastic gradient descent}
                \EndFor
                \State
                \State \textbf{Step 3:} Aggregation
                \For{each cluster $j=1,...,k$}
                    \State $r \gets 0$
                    \For{each neighbor $m \in \mathcal{N}_{i,j}$}
                    \State $r \gets r + 1$
                    \State $\theta_{i,j} \leftarrow \frac{r}{r+1} \theta_{i,j} + \frac{1}{r+1} \theta_{m,j}$ \\
                    \Comment{running average for each cluster}
                    \EndFor
                \EndFor
            \EndFor
        \EndFor
        \end{algorithmic}
\end{algorithm}

DFCA starts with initialization of the model parameters and then iterates three steps: (1) Cluster Assignment, (2) Local Updates, and (3) Decentralized Aggregation. Steps (1) and (2) are similar to existing training loss based, client-side clustered FL algorithms (\cite{IFCA, lin2025fedspd, ClusterFLsurvey}). Step (3) enables decentralized learning.

\textbf{Initialization. } \quad Before detailing the three iterative steps, we explain how the model parameters $\theta_{i,j}$ are initialized.  We consider two variants:  (i) with the \emph{global initialization} method (DFCA-GI), all $k$ models are centrally generated and published via broadcast (or initialized locally using the same seed) before the first iteration, so that every client holds the same model parameters at the beginning. (ii) For the \emph{local initialization} method (DFCA-LI), all clients start on different parameters, i.e., each client can initialize the models locally.

\begin{figure}[ht]
    \centering
    \includegraphics[width=0.7\linewidth]{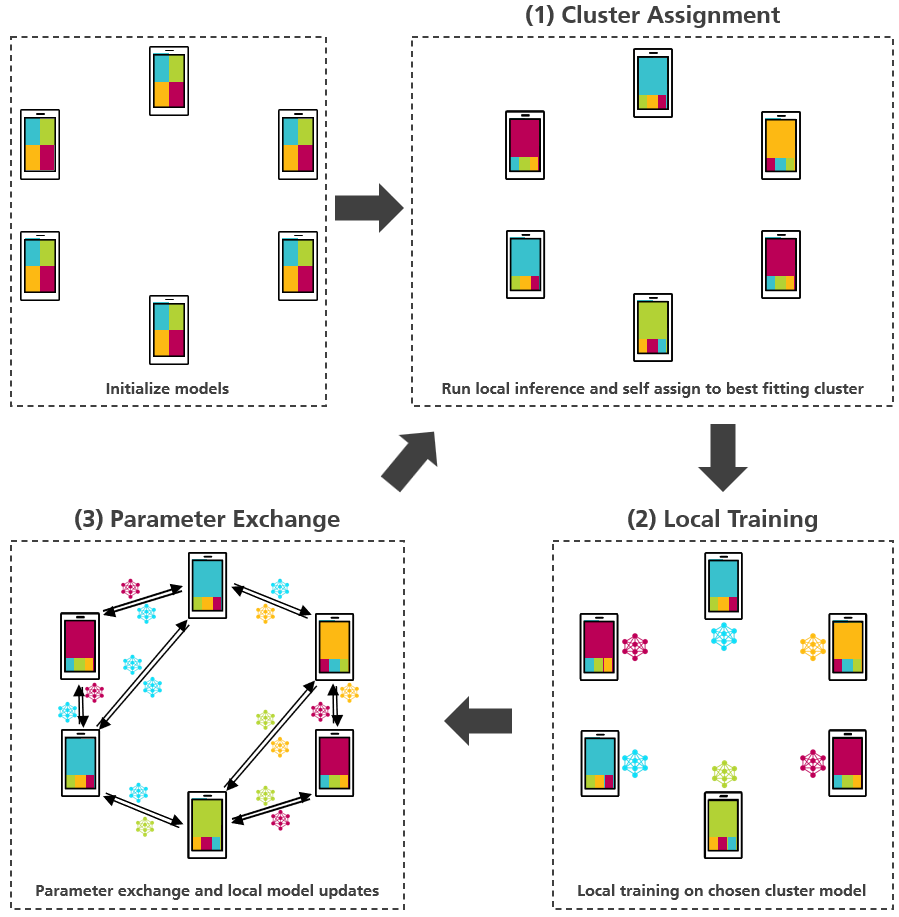}
    \caption{After initialization, DFCA iterates three steps: (1) cluster assignment, (2) local training, and (3) parameter exchange}
    \label{fig:dfca}
\end{figure}

\subsection{Cluster Assignment}

Every client $i$ is assigned to cluster $c(i)\in[k]$ through inference on the current parameters $\theta_{i,j}$. More formally, we update $c(i)$ to be the argmin of the local client loss,
\begin{equation}
   c(i)\leftarrow\underset{j}{\argmin}F_\text{client}(\theta_{i,j}, D_i).
\end{equation}
Hereby, the overall loss $F_\text{global}$ is non-increasing. These cluster assignments are repeated at the start of each training loop.
\subsection{Local Update}

The local update at client $i$ runs several epochs at the client-level using (stochastic) gradient descent on the local data $D_i$ with respect to $\theta_{i,c(i)}$:
\begin{equation}
    \theta_{i, c(i)} \leftarrow  \theta_{i, c(i)} - \gamma \nabla F_\text{client}(\theta_{i,c(i)}, D_i)
\end{equation}
(with learning rate $\gamma$), i.e., we only modify the parameters of the assigned cluster $c(i)$.
Again, the gradient descent ensures that the global loss $F_\text{global}$ is decreasing
(at least in expectation).

\subsection{Decentralized aggregation (a.k.a. communication step)}

The goal of our algorithm is that at the end, all clients hold \emph{all} $k$ trained models.  Thus, limiting the communication to neighbors within the \emph{same} cluster would be suboptimal.  Instead, all clients exchange their parameters with \emph{all} of their neighbors according to the graph and locally average the models. More formally, client $i\in M$ receives parameters from all neighbors in its neighborhood $\mathcal{N}_i$.  
To maintain cluster-specific updates in a sparse decentralized network, clients receive models from their neighbors but only send out the model parameters $\theta_{i,c(i)}$ that they trained themselves in the previous step.

To specify the aggregation equations, we split the neighbors of client $i$ according to their cluster assignments:
\begin{equation}
    \mathcal{N}_{i,j} := \{m \, | \, m\in\mathcal{N}_i \text{ and } c(m)=j\} \subset \mathcal{N}_i
\end{equation}
(for $i\in[N]$ and $j \in[k]$).
In this phase, the clients update the parameter sets for all clusters, not only the one of their assigned cluster $c(i)$.

\paragraph{Batch aggregation}
Next, we define the batch update (synchronous), which assumes that all neighbors $m$ have reported their current models $\theta_{m,j}$:
\begin{equation}
    \theta_{i,j} \leftarrow \frac{1}{|\mathcal{N}_{i,j}|+1}\left(\theta_{i,j} + \sum_{m\in\mathcal{N}_{i,j}} \theta_{m,j}\right)
    \label{eq:batch}
\end{equation}
(for $i \in[N]$ and $j \in[k]$).
\paragraph{Sequential aggregation}
While batch aggregation is the perfect scenario, in practice, neighbors might report their updates asynchronously, and we can never be sure whether a client has disconnected or not. Thus we need  sequential averaging that is robust against failing clients and random arrival times. The basic idea is to replace the averaging in Eq.~\ref{eq:batch} with an online version (a.k.a. running average): we start with the local parameter value $\theta_{i,j}$ and update it as the messages from the other clients come in. Assuming $r$ neighbors have already reported their updates for cluster $j$, we update $\theta_{i,j}$ with:
\begin{equation}
    \theta_{i,j} \leftarrow \frac{r}{r+1} \theta_{i,j} + \frac{1}{r+1} \theta_{m,j} \quad\quad \text{for }r \in[|\mathcal{N}_{i,j}|]
\end{equation}
(for $i \in[N]$ and $j \in[k]$).

Our sequential aggregation naturally supports asynchronous updates, allowing each client to integrate neighbor models immediately as they arrive, which can improve efficiency and reduce idle time in fully distributed deployments. This approach is also memory efficient, as it only requires storing the current estimate per model rather than all neighbor updates. Moreover, using a running average ensures that each incoming model contributes proportionally to the aggregated model, providing a stable and principled approximation of the full batch aggregation even in dynamic and sparse networks.


\subsection{Computational and communication overhead.}
At each iteration, a client performs local training only on its currently assigned cluster-specific model, resulting in a computational cost comparable to standard decentralized SGD. Although each client stores all $k$ models, inference-based cluster assignment requires only forward passes and results in overhead compared to local training. Communication is fully decentralized and limited to parameter exchanges with neighboring clients without relying on global broadcasts and, although not specifically tested in our experiments, global synchronization. The per-round communication cost therefore scales linearly with the neighborhood size and the number of clusters $k$. In typical IoT deployments, where both $k$ and neighborhood sizes are relatively small, this results in manageable overhead, making DFCA suitable for scenarios in which low communication cost is a key requirement \citep{jsan14010009}.

\section{Convergence Summary}

\begin{table}[t]
\centering
\caption{Notation and parameters used in DFCA convergence analysis}
\label{tab:notation}
\begin{tabular}{ll}
\hline
\textbf{Symbol} & \textbf{Description} \\
\hline
$N$ & Number of clients in the network \\
$M=\{1,\dots,N\}$ & Set of clients \\
$k$ & Number of clusters \\
$[k]$ & Cluster index set $\{1,\dots,k\}$ \\
$\mathcal S_j$ & Set of clients belonging to cluster $j$ \\
$\mathcal D_j$ & Data distribution of cluster $j$ \\
$D_i$ & Local dataset of client $i$ \\
$d$ & Data sample \\
$\mathcal L(\theta,d)$ & Sample-wise loss function \\
$F_{\text{client}}(\theta,D_i)$ & Empirical loss of client $i$ \\
$F_{\text{cluster}}(j)$ & Objective of cluster $j$ \\
$F_{\text{global}}$ & Global objective $\sum_{j=1}^k F_{\text{cluster}}(j)$ \\
\hline
$\theta_{i,j}^t$ & Model of client $i$ for cluster $j$ at round $t$ \\
$\Theta_j^t$ & Stacked models $(\theta_{1,j}^t,\dots,\theta_{N,j}^t)$ \\
$\bar\theta_j^t$ & Network-wide average model for cluster $j$ \\
$c_t(i)$ & Cluster assignment of client $i$ at round $t$ \\
$c_\star(i)$ & Ground-truth cluster assignment of client $i$ \\
\hline
$G=(M,E)$ & Communication graph \\
$\mathcal N_i$ & Neighborhood of client $i$ in $G$ \\
$W$ & Mixing matrix (synchronous gossip) \\
$W_t^{(j)}$ & Time-varying mixing matrix for cluster $j$ \\
$\lambda$ & Consensus contraction factor (synchronous case) \\
$\tilde\lambda$ & Windowed contraction factor (asynchronous case) \\
$B$ & Gossip window size in asynchronous setting \\
\hline
$\mathrm{Disp}_j^t$ & Disagreement for cluster $j$ at round $t$ \\
$E_t$ & Disagreement matrix $X_t - \mathbf{1}\bar x_t^\top$ \\
\hline
$\gamma$ & Learning rate (step size) \\
$g_{i,j}(\theta)$ & Stochastic gradient at client $i$, cluster $j$ \\
$L$ & Smoothness constant \\
$\sigma^2$ & Variance bound of stochastic gradients \\
$\mu$ & PL constant (when PL condition holds) \\
$\delta$ & Cluster separability margin (Assumption A5) \\
$\tau$ & Stabilization time of assignments \\
\hline
\end{tabular}
\end{table}

We briefly summarize the convergence properties of DFCA.  Full proofs are deferred to Appendix~\ref{app:convergence}.


\paragraph{Setup}
Each client stores all $k$ models $\{\theta_{i,j}^t\}$ here with index $t$ for the round. Each round executes three steps: (i) cluster assignment by local inference, (ii) local stochastic gradient descent on the assigned model, and (iii) decentralized aggregation with neighbors. Aggregation is carried out via gossip (either synchronous averaging or sequential running averages), which preserves the network-wide average and contracts disagreement among clients.
For cluster $j$, define the stacked vector $\Theta_j^t=(\theta_{1,j}^t,\dots,\theta_{N,j}^t)$ and the network average $\bar \theta^t_j=\tfrac1N\sum_{i=1}^N \theta^t_{i,j}$.
We measure per-cluster \emph{disagreement} by
\(
\mathrm{Disp}^t_j=\frac{1}{N}\sum_{i=1}^N \|\theta^t_{i,j}-\bar \theta^t_j\|^2.
\)
The three steps of one DFCA update round can be written as:
\begin{enumerate}
\item \textbf{Assignment:} \(c_t(i)=\arg\min_{j\in[k]} F_{\text{client}}(\theta^t_{i,j},D_i)\).
\item \textbf{Local descent (assigned index only):}
\begin{equation} 
\theta^{t+\frac12}_{i,c_t(i)}=\theta^t_{i,c_t(i)}-\gamma\,g_{i,c_t(i)}(\theta^t_{i,c_t(i)}),
 \end{equation}
with $\theta^{t+\frac12}_{i,j}=\theta^t_{i,j}\ \ (j\neq c_t(i))$ and stochastic gradient $g_{i,j}$.
\item \textbf{Decentralized aggregation (all $j$):}
\begin{equation} 
\theta^{t+1}_{i,j}=\sum_{m\in\{i\}\cup \mathcal N_i} w^{(j)}_{im,t}\,\theta^{t+\frac12}_{m,j},
 \end{equation}
where $W^{(j)}_t=(w^{(j)}_{im,t})$ respects $G$, is row-stochastic, and is doubly-stochastic in the synchronous (batch) case. In the sequential/asynchronous case, $W^{(j)}_t$ is time-varying with standard joint-connectivity.
\end{enumerate}

We adopt the following standard assumptions.

\begin{enumerate}
[label=(A\arabic*)]
\item \textbf{Smoothness.} For all $i,j$, $F_{\text{client}}(\cdot,D_i)$ is $L$-smooth.
\item \textbf{Noise.} Unbiased stochastic gradients with bounded variance:
\(
\mathbb E[g_{i,j}(\theta)\mid \theta]=\nabla F_{\text{client}}(\theta,D_i),\ 
\mathbb E\|g_{i,j}(\theta)-\nabla F_{\text{client}}(\theta,D_i)\|^2\le \sigma^2.
\)
\item \textbf{Graph mixing.} In the synchronous case there exists a symmetric, doubly-stochastic $W$ respecting $G$ with spectral gap $1-\lambda>0$ such that
\(
\|XW - \mathbf 1 \bar x^\top\|\le \lambda\,\|X-\mathbf 1 \bar x^\top\|
\)
for any row-stacked $X$.
In the asynchronous case, $\{W^{(j)}_t\}$ are row-stochastic, edges are repeatedly activated with bounded delays, and there exists a window $B$ and $\tilde \lambda\in(0,1)$ such that over any $B$ consecutive rounds disagreement contracts by $\tilde \lambda$. (extensive formulation can be found in Appendix \ref{ext:A3})
\item \textbf{Objective curvature.} Either (PL) each $F_{\text{cluster}}(j;\cdot)$ satisfies the $\mu$-Polyak--\L{}ojasiewicz (PL) inequality, or (Cvx) each is convex.
\item \textbf{Separability (IFCA-style).} There exists $\delta>0$ such that, in a neighborhood of the cluster minimizers $\{\theta^\star_j\}_{j=1}^k$, the argmin-of-loss assignment selects the true cluster:
\begin{equation} 
\mathbb E_{d\sim \mathcal D_{c(i)}}\!\big[\mathcal L(\theta_{i,c(i)},d)\big]
\le \min_{j\neq c(i)} \mathbb E_{d\sim \mathcal D_{c(i)}}\!\big[\mathcal L(\theta_{i,j},d)\big]-\delta.
 \end{equation}
\end{enumerate}

On a high-level, the analysis combines two ingredients:
\begin{enumerate}
    \item \emph{Cluster assignment:} Choosing the best-fitting model index per client never increases the global loss, and after sufficient descent the assignments stabilize to the ground-truth clusters.
    \item \emph{Local descent + gossip:} Gradient descent decreases the cluster objectives, up to stochastic noise and a \emph{disagreement penalty}. Gossip averaging preserves the average model and contracts disagreement at a rate governed by the graph spectral gap.
\end{enumerate}
Together, these steps imply that DFCA behaves like $k$ independent instances of decentralized SGD, one per cluster, after a finite burn-in.

\begin{theorem}[Convergence of DFCA]\label{thm:dfca}
Assume (A1)--(A5), choose $\gamma\le c/L$ for a small numerical constant $c$, and let $\lambda$ (resp.\ $\tilde \lambda$) be the consensus factor in the synchronous (respectively asynchronous) case. Then:
\begin{enumerate}[label=(\roman*)]
\item (\emph{Pre-stabilization}) $F_{\text{global}}^t$ is non-increasing in expectation across assignment and local steps. The disagreements $\{\mathrm{Disp}^t_j\}$ remain bounded and contract at rate $\lambda$ (or $\tilde \lambda$ over windows).
\item (\emph{Stabilization}) There exists $\tau<\infty$ such that $c_t(i)=c_\star(i)$ for all $t\ge \tau$.
\item (\emph{Post-stabilization}) For $t\ge \tau$, DFCA is $k$ independent copies of decentralized SGD on $F_{\text{cluster}}(j)$.
\begin{itemize}
\item Under (PL) for some $\mu>0$,
\begin{equation}
\begin{split} 
\mathbb E\!\left[F_{\text{global}}^{\tau+T}-F_{\text{global}}^\star\right] \ \le\ \\
(1-\mu\gamma/2)^T\,C_0
\ +\ O\!\Big(\tfrac{\gamma\sigma^2}{\mu}\Big)
\ +\ O\!\Big(\tfrac{\gamma L}{1-\lambda}\,\sigma^2\Big),
\end{split}
\end{equation}
with $C_0$ depending on the gap at $t=\tau$; in async, replace $(1-\lambda)$ by the windowed $(1-\tilde \lambda)$.
\item Under (Cvx),
\begin{equation}
\begin{split} 
\frac1T \sum_{t=\tau}^{\tau+T-1}\sum_{j=1}^k \mathbb E\|\nabla F_{\text{cluster}}(j;\bar \theta^t_j)\|^2 \ \le\ \\
O\!\Big(\frac{F_{\text{global}}^\tau-F_{\text{global}}^\star}{\gamma T}\Big)
+O(\gamma L\sigma^2)
+O\!\Big(\tfrac{\gamma L}{1-\lambda}\sigma^2\Big),
\end{split}
\end{equation}
and choosing $\gamma=\Theta(1/\sqrt{T})$ yields the usual $O(1/\sqrt{T})$ rates (with the consensus penalty).
\end{itemize}
\end{enumerate}
\end{theorem}

\paragraph{Takeaway}
DFCA converges at essentially the same rate as decentralized SGD, up to an additional term reflecting network connectivity. Crucially, all clients obtain all $k$ cluster models despite decentralized, asynchronous communication. The appendix provides a detailed proof by combining IFCA’s cluster-assignment arguments with standard decentralized SGD analyses.

\section{Experiments}
Next, we present our experiments with DFCA in practical learning settings. As common in the clustered FL literature \citep{IFCA, lin2025fedspd, FEDSOFT}, we conduct experiments on the MNIST \citep{MNIST}, EMNIST \citep{emnist}, CIFAR-10 \citep{cifar}, and FEMNIST \citep{femnist} datasets, while applying rotations to the data to create different distributions. Our method is compared to the decentralized soft-clustering method FedSPD \citep{lin2025fedspd} and the optimized Decentralized Federated Averaging algorithm DFedAvgM \cite{sun2021decentralizedfederatedaveraging}. IFCA \citep{IFCA} serves as the centralized baseline. After providing results for additional experiments with different connection probabilities, we discuss the communication efficiency and analyze the results of our experiments.

\subsection{Experimental Setting}
\textbf{EMNIST:}
For the training with EMNIST \citep{emnist} (balanced split), we use $N=100$ clients for $k=2$ and $N=200$ for $k=4$ clusters and simulate two or four different data distributions by augmenting the datasets, applying 0, 180 or 0, 90, 180, 270 degree rotations to the data. The Convolutional Neural Network (CNN) used for training contains two convolutional layers, each followed by a ReLU activation function, a max-pool layer, and a batch normalization layer. The models are trained for $\tau=5$ local epochs with a learning rate of $\gamma=0.1$, using Stochastic Gradient Descent over $T=150$ global iterations. For the connection between clients, we use the adjacency matrix of an Erdős–Rényi graph with a connection probability of $0.15$. All experiments are run on five random seeds, with the metric values being averaged over all runs.

\textbf{MNIST:}
The training with MNIST \citep{MNIST} is conducted on $N=240$ clients and $k=4$ clusters and data distributions (0, 90, 180, 270 degree rotations). We use a simple Multilayer Perceptron (MLP) with one hidden layer of size 2048 followed by a ReLU activation function. The other training parameters stay consistent with the EMNIST experimental setting, with the exception of reducing the connection probability to $0.1$.

\textbf{CIFAR-10:}
 The setup for our experiments with the CIFAR-10 \citep{cifar} dataset is similar to the EMNIST setup. We train with an identical CNN architecture over $N=50$ clients and $k=2$ clusters. We change the learning rate to $\gamma=0.25$ and the graph connection probability to $0.2$.

\textbf{FEMNIST:}
To test the algorithm in settings with even higher heterogeneity, we conducted experiments on the FEMNIST \citep{femnist} dataset. The training is done on $N=400$ clients, who each get data from one distinct writer, with $k=3,4,10$ clusters with a graph connection probability of $0.2$ and all other parameters equal to the MNIST experiments.

\begin{figure*}[!t]
    \centering
    \includegraphics[width=0.7\linewidth]{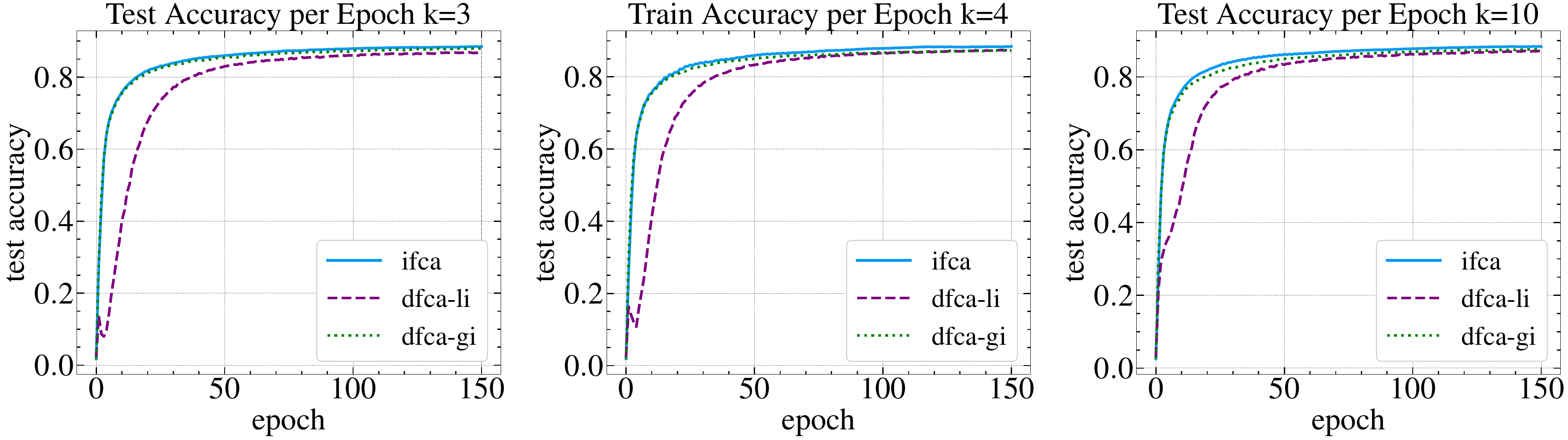}
    \caption{Graphs for FEMNIST experiments with different values for $k$}
    \label{fig:femnist}
\end{figure*}

\begin{figure}
    \centering
    \includegraphics[width=\linewidth]{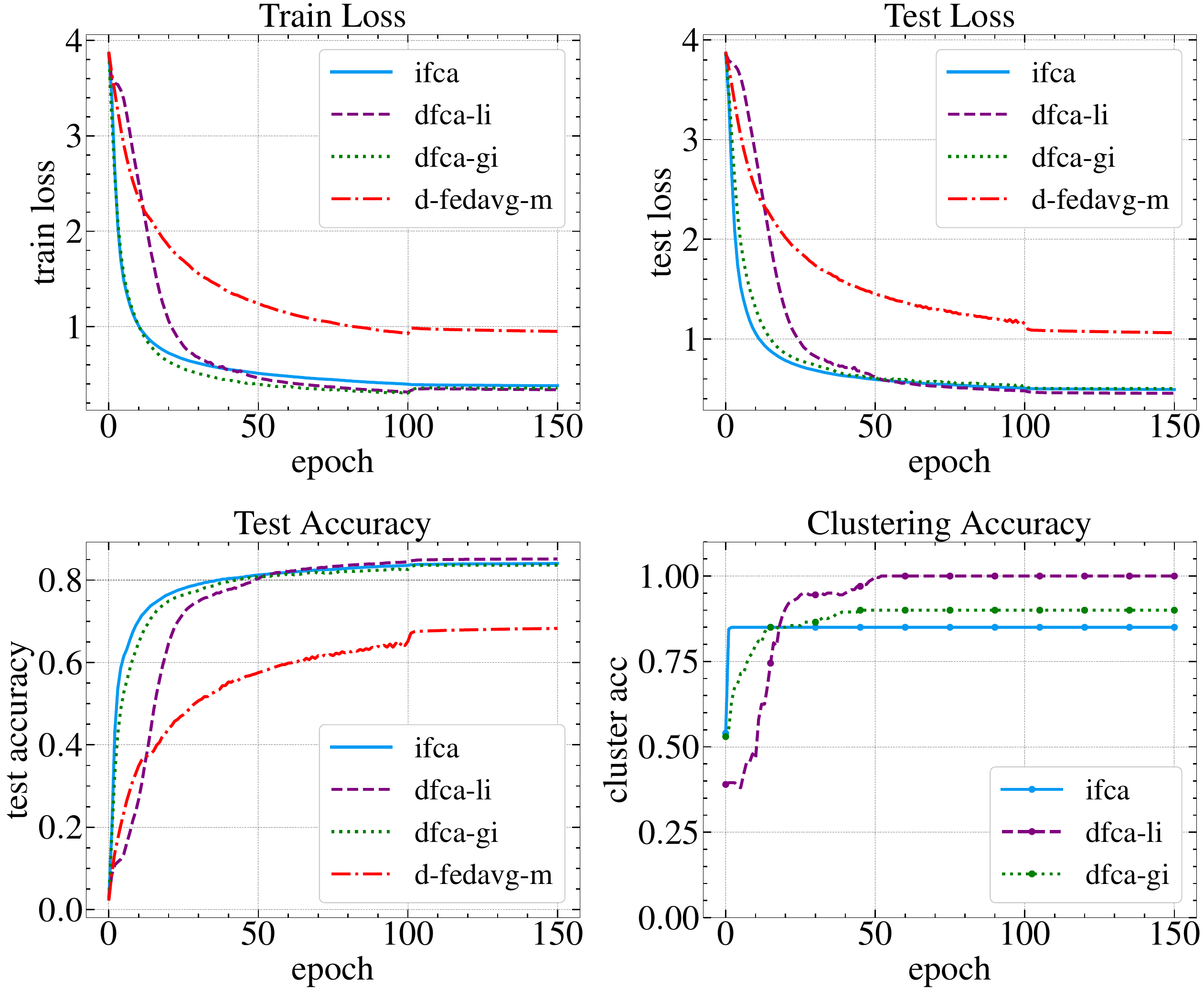}
    \caption{Plots for EMNIST experiments with edge connection probability of $0.15$, $N=200$, $k=4$ on IFCA, DFCA LI and GI and D-FedAvg-M}
    \label{fig:emnist_plots}
\end{figure}

\begin{figure}
    \centering
    \includegraphics[width=\linewidth]{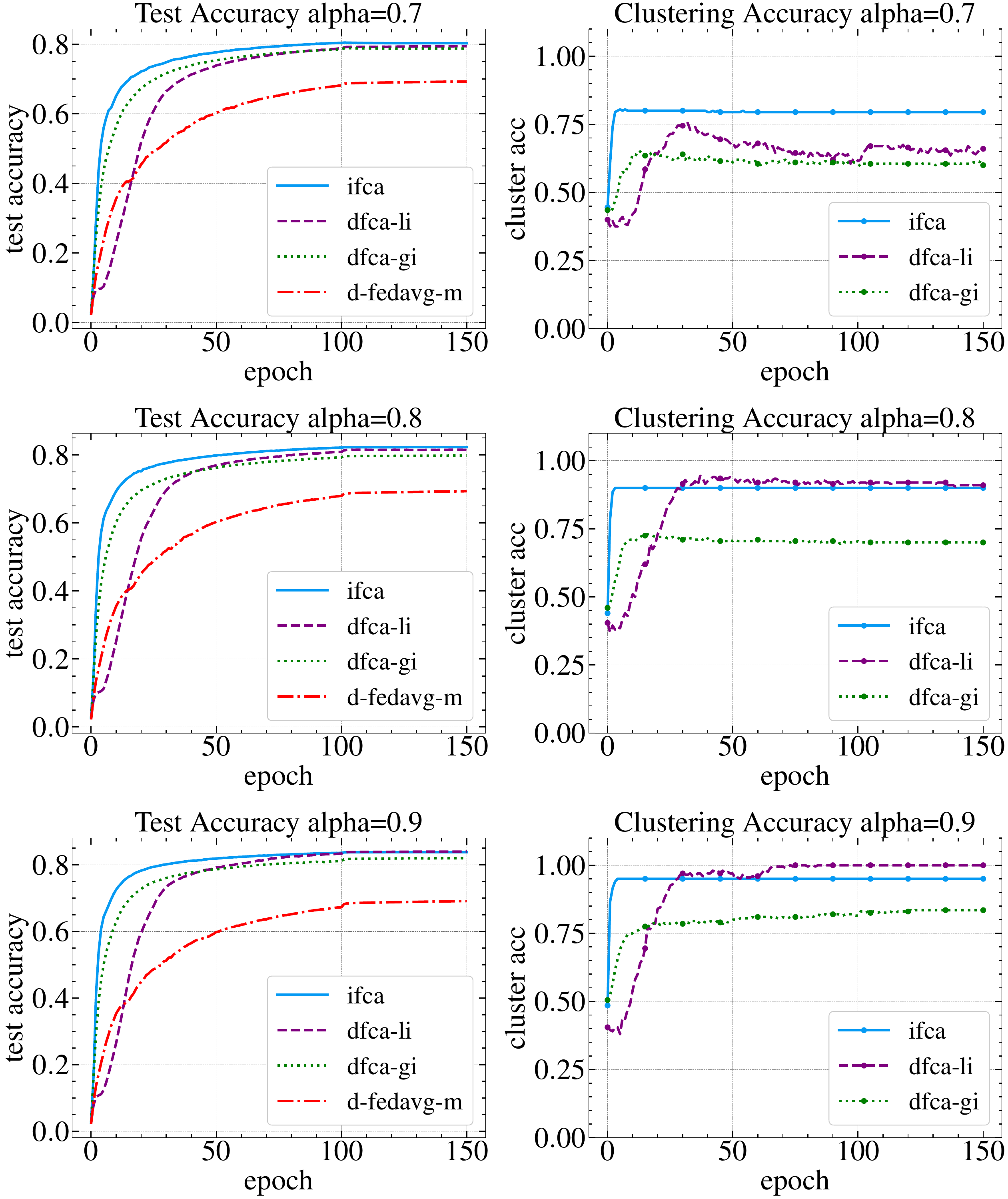}
    \caption{Results for cluster-inconsistent noisy feature skew experiments on EMNIST with connection probability of $15\%$, $k=4$, $\alpha=0.8$ and $\alpha=0.9$}
    \label{fig:alpha_skew_rand}
\end{figure}

\begin{figure}
    \centering
    \includegraphics[width=\linewidth]{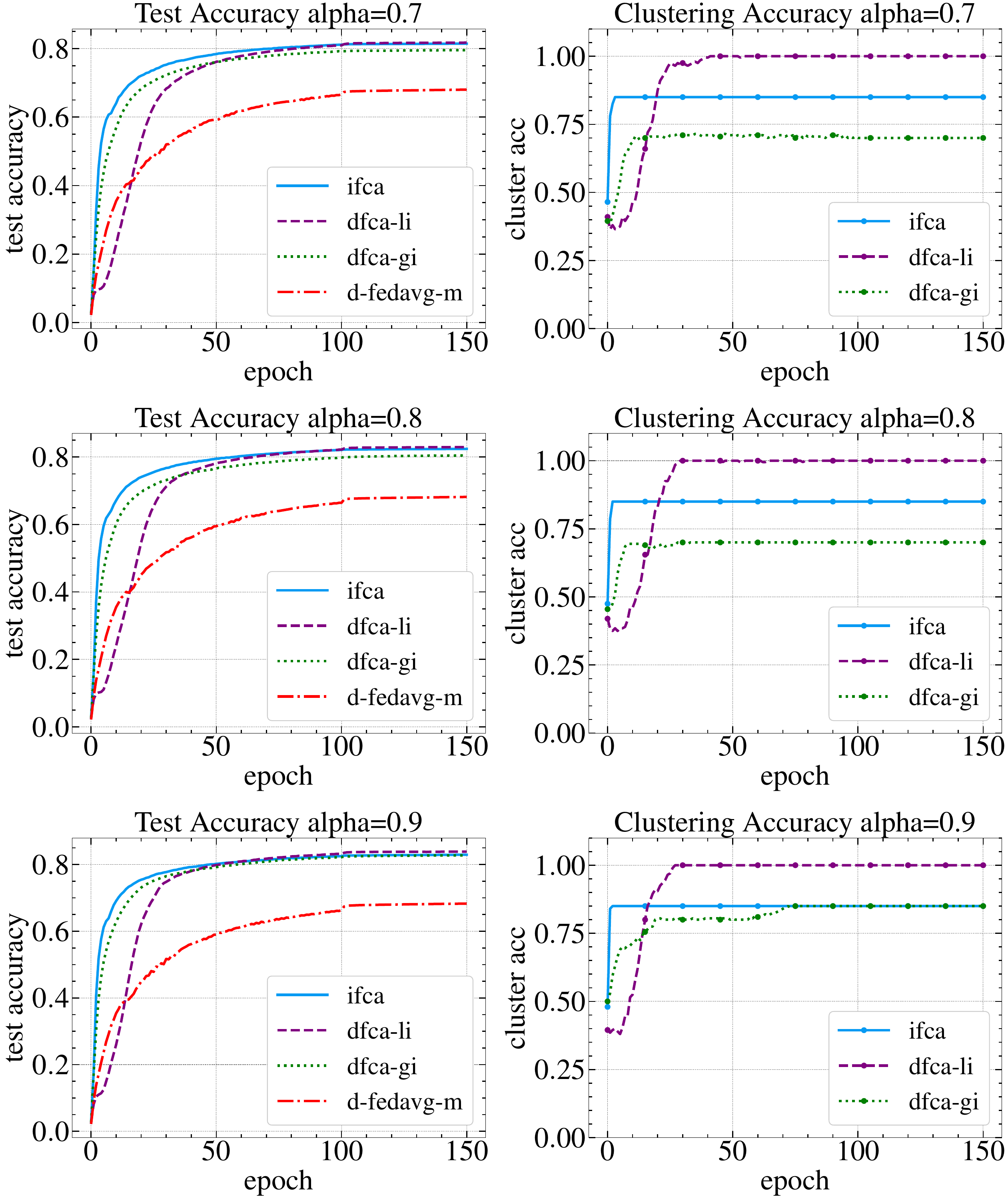}
    \caption{Results for cluster-consistent noisy feature skew experiments on EMNIST with connection probability of $15\%$, $k=4$, $\alpha=0.8$ and $\alpha=0.9$}
    \label{fig:alpha_skew}
\end{figure}

\subsection{General Accuracy Analysis}
Our experiments demonstrate that DFCA consistently outperform the decentralized baselines FedSPD and DFedAvgM while achieving accuracy comparable to the centralized IFCA algorithm (Table \ref{table:accuracy_one_r}). The reason for FedSPD underperforming is that it represents a decentralized FL approach that adapts a single local model per client via regularization-based decentralized averaging. In the presence of clustered heterogeneity, this limits effective model sharing among clients with similar data distributions, whereas DFCA explicitly maintains shared cluster-specific models. The plots in Figure \ref{fig:emnist_plots} additionally show DFCA-GI converging at a similar rate as IFCA, while DFCA-LI converges slower but steeper than the other two methods. In more heterogeneous settings with larger numbers of clients (MNIST, EMNIST, FEMNIST; Table \ref{table:accuracy_two_r}), DFCA maintains competitive performance, indicating that the sequential aggregation strategy effectively preserves cluster-specific models even as heterogeneity increases. IFCA's unusually high standard deviation for the EMNIST experiments with $k=4$ occurs because IFCA detected only three clusters in one of its five runs. The topic of misclusterings is further discussed in the next sections.

\begin{table*}[!t]
\caption{Results of experiments with EMNIST ($N=100$ clients) and CIFAR-10 ($N=50$ clients)}
	\centering
	\begin{tabular}{ |c|c|c|c|c|c| } 
	\hline
     & \multicolumn{4}{c|}{DFL} & \multicolumn{1}{c|}{CFL} \\ 
    \hline
	 Dataset & DFCA-GI (ours) & DFCA-LI (ours) & FedSPD & DFedAvgM & IFCA \\ 
     \hline
    MNIST & 93.7 $\pm$ 0.07 & 92.9 $\pm$ 0.06 & 86.2 $\pm$ 1.52 & 91.4 $\pm$ 0.21 & \textbf{93.9 $\pm$ 0.06} \\
    \hline
    EMNIST & \textbf{85.7 $\pm$ 0.13} & 85.3 $\pm$ 0.09 & 79.7 $\pm$ 0.92 & 73.5 $\pm$ 1.19 & \textbf{85.7 $\pm$ 0.11} \\
    \hline
    CIFAR-10 & 81.5 $\pm$ 0.40 & 80.4 $\pm$ 0.22 & 78.9 $\pm$ 0.23  & 76.0 $\pm$ 0.96 & \textbf{82.5 $\pm$ 0.11 } \\
    \hline
    \end{tabular}
    \label{table:accuracy_one_r}
\end{table*}

\begin{figure}
    \centering
    \includegraphics[width=0.8\linewidth]{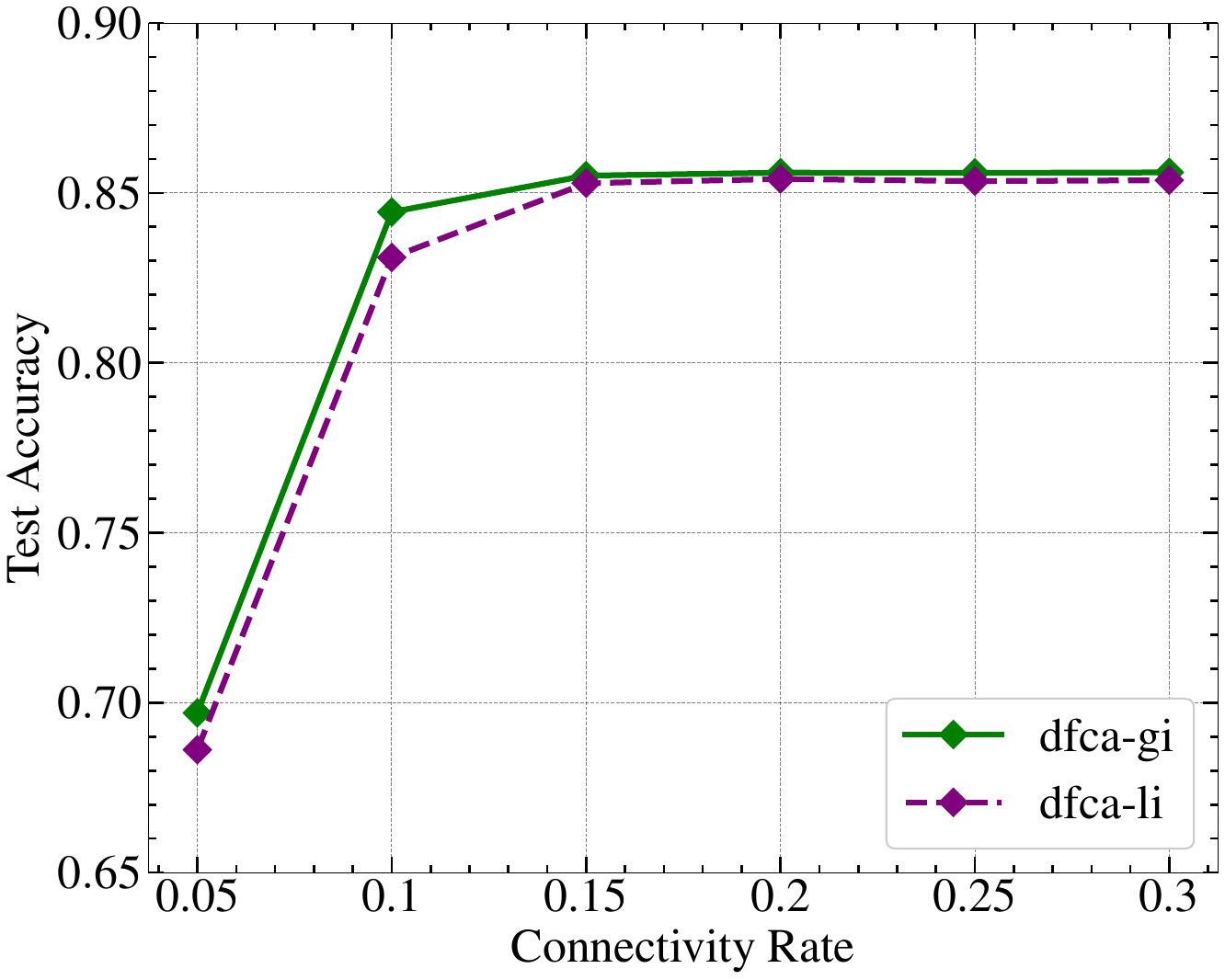}
    \caption{Test Accuracy of DFCA-LI and DFCA-GI under different connectivity settings (EMNIST, $k=2$, $N=100$)}
    \label{fig:connectivity}
\end{figure}

\begin{table}[!ht]
\caption{Additional comparisons with IFCA for $k=4$, $N=200$ on MNIST, $N=100$ on EMNIST, and $N=400$ on FEMNIST.}
	\centering
	\begin{tabular}{ |c|c|c|c| } 
	\hline
     & \multicolumn{2}{c|}{DFL} & \multicolumn{1}{c|}{CFL} \\ 
    \hline
	Dataset & DFCA-GI (ours) & DFCA-LI (ours) & IFCA \\ 
     \hline
    MNIST & 92.8 $\pm$ 0.63 & 92.4 $\pm$ 0.22 & \textbf{93.1 $\pm$ 0.73} \\
    \hline
    EMNIST & 83.6 $\pm$ 2.10 & \textbf{85.1 $\pm$ 0.10} & 84.0 $\pm$ 1.82  \\
    \hline
    FEMNIST & 87.1 $\pm$ 0.30 & 86.4 $\pm$ 0.15 & \textbf{88.2 $\pm$ 0.11}  \\
    \hline
    \end{tabular}   
    \label{table:accuracy_two_r}
\end{table}

\textbf{Insights. } \quad In DFL, the way clients exchange model updates plays a crucial role in both convergence and efficiency. Beyond simple averaging, enabling clustered FL in decentralized networks is particularly valuable, as it allows clients with heterogeneous data to specialize in distinct model clusters without relying on a central coordinator. The general advantages of DFL, such as improved scalability, resilience to single points of failure, and better suitability for bandwidth-limited or peer-to-peer networks, have already been highlighted in prior works \citep{lalitha2018fully, lalitha2019peer, DFL}. The results in Tables \ref{table:accuracy_one_r} and \ref{table:accuracy_two_r} show that DFCA not only outperforms the decentralized baselines but also does not fall short when compared to centralized IFCA.  Despite evidence that DFL lags behind CFL \citep{sun2024which}, we reduce the accuracy difference to about ~1\% in CFL’s favor, including in non-IID and low-connectivity settings.

\subsection{Noisy Feature Skew Analysis}
We further assess the robustness of DFCA under more challenging heterogeneous feature distributions, by introducing two additional scenarios and varying the dominance ratio $\alpha \in \{0.7, 0.8, 0.9\}$ (with $\alpha=1$ matching the main setup). Such feature skews can be seen in IoT and distributed sensing applications, where devices operate under heterogeneous and dynamically changing conditions (e.g., sensor noise, calibration differences, environmental effects, or device-specific preprocessing), leading to partially misaligned feature transformations across clients. The setup is distributed in two different styles:

1. \textbf{Cluster-consistent skew}: Clients belonging to the same cluster share the same dominant transformation, while retaining a non-dominant share $(1-\alpha)$ of samples with a different transformation. This setting models coherent subpopulations with strong intra-cluster similarity.

2. \textbf{Cluster-inconsistent skew}: Each client receives its non-dominant transformation independently at random, introducing a more challenging level of noise to the clusters. This represents heterogeneous environments in which local feature distributions are not aligned across clients.

The results of our experiments are illustrated in Table \ref{table:accuracy_alpha_r} and Figure \ref{fig:alpha_skew_rand}.

\textbf{Insights. } \quad DFCA-LI is more robust than both DFCA-GI and IFCA under more noisy feature distribution settings, where clusters are less clear. The method clusters more accurately under most settings, and outperforms all other decentralized methods and in most cases even the centralized baseline IFCA in terms of accuracy. The accuracy gains are largely due to the instability observed in IFCA and DFCA-GI, whose clustering assignments degrade substantially under higher noise levels (see Figure \ref{fig:alpha_skew_rand} and Table \ref{table:accuracy_alpha_r}).

\begin{table*}[!t]
\caption{Experiments with EMNIST ($N=200$ clients) and on noisy feature distributions}
	\centering
	\begin{tabular}{ |c|c|c|c|c|c| } 
	\hline
     & \multicolumn{3}{c|}{DFL} & \multicolumn{1}{c|}{CFL} \\ 
    \hline
	 Setting & DFCA-GI (ours) & DFCA-LI (ours) & DFedAvgM & IFCA \\ 
     \hline
    Consistent, $\alpha=0.7$ & 80.5 $\pm$ 1.11 & \textbf{82.9 $\pm$ 0.11} & 68.2 $\pm$ 0.30 & 82.4 $\pm$ 1.01 \\
    \hline
    Consistent, $\alpha=0.8$ & 80.5 $\pm$ 1.11 & \textbf{82.9 $\pm$ 0.11} & 68.2 $\pm$ 0.30 & 82.4 $\pm$ 1.08 \\
    \hline
    Consistent, $\alpha=0.9$ & 82.8 $\pm$ 1.05 & \textbf{83.9 $\pm$ 0.27} & 68.3 $\pm$ 0.35 & 82.9 $\pm$ 1.78 \\
    \hline
    Inconsistent, $\alpha=0.7$ & 79.8 $\pm$ 1.47 & 81.5 $\pm$ 1.64 & 69.3 $\pm$ 0.24 & \textbf{82.3 $\pm$ 1.58}   \\
    \hline
    Inconsistent, $\alpha=0.8$ & 79.8 $\pm$ 1.46 & 81.5 $\pm$ 1.64 & 69.3 $\pm$ 0.24 & \textbf{82.3 $\pm$ 1.57}  \\
    \hline
    Inconsistent, $\alpha=0.9$ & 82.0 $\pm$ 1.64 & \textbf{84.0 $\pm$ 0.15} & 69.1 $\pm$ 0.31 & 83.8 $\pm$ 1.73  \\
    \hline
    \end{tabular}
    \label{table:accuracy_alpha_r}
\end{table*}

\subsection{Connectivity Analysis}
Figure \ref{fig:connectivity} shows the test accuracy of DFCA-LI and DFCA-GI under different, fixed connectivity settings on EMNIST. There, we can observe that a connectivity of $0.15$ is sufficient and the test accuracy does not change significantly when further increasing the connectivity rate. In settings with connectivity probabilities below $0.1$, DFCA-LI attains slightly lower accuracies than DFCA-GI, which can be attributed to its slower convergence caused by the additional time required for clustering, as seen in Figures \ref{fig:emnist_plots} and \ref{fig:alpha_skew_rand}.

\textbf{Insights. } \quad DFCA leverages a sequential running average to integrate neighbor updates efficiently, avoiding the need to store all incoming models and allowing updates to proceed asynchronously as they arrive. As a result, it improves scalability and robustness to network sparsity, while still achieving accuracy comparable to the centralized IFCA baseline.

\subsection{Misclusterings and finding the right k}
When the number of data distributions is known a priori, as often considered in the literature \citep{IFCA}, misclusterings are rather rare. During the experiments, we saw that DFCA is more robust to misclusterings than IFCA, as IFCA sometimes struggled with finding all correct clusters for 4 data distributions. IFCA only found three of four clusters in one to two out of five runs, which resulted in lower accuracy than DFCA in the EMNIST results. However, even when IFCA misclusters, the performance remains competitively high with only marginal differences in accuracy.

\textbf{Insights. } \quad In the case of natural feature skew, like in FEMNIST, we started with $k=10$ for $N=400$ clients, where each client receives data from one distinct writer. In the experiments, we saw that all methods (DFCA-LI, DFCA-GI, IFCA) detected between three and five clusters. This occurs because the feature differences between writers are often subtle, making it difficult for the models to reliably distinguish fine-grained variations. As a result, they tend to generalize across writers and merge them into fewer effective clusters. The comparisons of performance for $k=3,4,10$ can be seen in Table \ref{table:accuracy_femnist_r} and in further plots in the appendix, where the differences in performance are not substantial. Using $k=3$ instead of $k=10$ substantially reduces the communication and memory cost, as less models need to be saved and maintained.

\begin{table}[!ht]
	\centering
\caption{Results on FEMNIST for different $k$ with $N=400$ clients, each holding data from a distinct writer}
	\begin{tabular}{ |c|c|c|c| } 
	\hline
     & \multicolumn{2}{c|}{DFL} & \multicolumn{1}{c|}{CFL} \\ 
    \hline
	Num Clusters & DFCA-GI (ours) & DFCA-LI (ours) & IFCA \\ 
     \hline
    $k=3$ & 87.9 $\pm$ 0.59 & 86.9 $\pm$ 0.89 & \textbf{88.5 $\pm$ 0.52} \\
    \hline
    $k=4$ & 87.3 $\pm$ 0.62 & 87.3 $\pm$ 0.27 & \textbf{88.4 $\pm$ 0.12}  \\
    \hline
    $k=10$ & 87.5 $\pm$ 0.76 & 87.1 $\pm$ 0.79 & \textbf{88.3 $\pm$ 0.43}  \\
    \hline
    \end{tabular}   
    \label{table:accuracy_femnist_r}
\end{table}

\section{Conclusion}
\paragraph{Conclusion}
In this work, we introduced DFCA, a fully serverless method inspired by IFCA, that allows cluster-specific models to emerge and propagate through heterogeneous, sparse peer-to-peer networks. By employing a sequential running-average aggregation scheme, DFCA leverages stable learning with high clustering accuracy in heterogeneous environments where centralized methods are impractical. Our experimental results demonstrate that DFCA achieves performance comparable to centralized IFCA while operating under decentralized communication constraints, and it consistently outperforms decentralized FedAvg with momentum and FedSPD.

\paragraph{Discussion}
Having demonstrated that DFCA effectively addresses decentralized clustered learning under heterogeneous data distributions, we now briefly discuss practical aspects that our design enables. The running-average update scheme can improve wall-clock efficiency in decentralized clustered FL. Unlike synchronous averaging, where all clients must wait for the slowest participant, running averages allow clients to incorporate neighbor updates at their own pace. Slow clients contribute less frequently but do not block progress for others, which is particularly advantageous in heterogeneous IoT or edge networks. This enables faster, more continuous progress and can reduce overall time-to-convergence. As our experiments did not include tests on random arrival times or communication delays, quantifying this effect is left for future work.

DFCA is well-suited for applications where data are inherently clustered and no central server is available or desirable, such as industrial IoT, sensor networks, or smart-city deployments. A deeper investigation of communication overheads and deployment characteristics is an important direction for future work.

\ifCLASSOPTIONcaptionsoff
  \newpage
\fi

\bibliography{references}
\bibliographystyle{plain}

\appendix
\section{Appendix}

\subsection{Convergence Analysis}
\label{app:convergence}
We provide a proof template that reuses standard ingredients from clustered FL (e.g., \cite{IFCA}) for the assignment and from decentralized SGD/gossip (e.g., \citep{lian2017,koloskova2020unified,boyd2006randomized,nedic2018network}) for communication. Throughout, expectations are with respect to the stochasticity of data sampling and any communication randomness.

\paragraph{Notation}
Clients are $M=\{1,\dots,N\}$, connected by an undirected graph $G=(M,E)$ with neighborhoods $\mathcal N_i$.
The $k$ cluster index set is $[k]=\{1,\dots,k\}$ and the (unknown) partition is $\{\mathcal S_1,\dots,\mathcal S_k\}$ with data distributions $\{\mathcal D_1,\dots,\mathcal D_k\}$.
Client $i$ stores parameters $(\theta_{i,1},\dots,\theta_{i,k})\in (\mathbb R^d)^k$.
For cluster $j$, define the stacked vector $\Theta_j=(\theta_{1,j},\dots,\theta_{N,j})$ and the network average $\bar \theta^t_j=\tfrac1N\sum_{i=1}^N \theta^t_{i,j}$.
The client loss is
$$F_{\text{client}}(\theta_{i,j},D_i)=\frac{1}{|D_i|}\sum_{d\in D_i}\mathcal L(\theta_{i,j},d),$$
$$F_{\text{cluster}}(j)=\sum_{i\in\mathcal S_j}F_{\text{client}}(\theta_{i,j},D_i),$$
$$F_{\text{global}}=\sum_{j=1}^k F_{\text{cluster}}(j).$$

We measure per-cluster \emph{disagreement} by
\begin{equation}
    \mathrm{Disp}^t_j=\frac{1}{N}\sum_{i=1}^N \|\theta^t_{i,j}-\bar \theta^t_j\|^2.
\end{equation}

\paragraph{Assumption 3 Extension}\label{ext:A3}
Let $X_t$ be the row-stacked matrix of client models (for any fixed cluster index $j$),
let $\bar x_t = \tfrac{1}{N}\mathbf{1}^\top X_t$ be the network average,
and define the disagreement
\begin{equation}
E_t = X_t - \mathbf{1}\bar x_t^\top.
\end{equation}

\textbf{Synchronous mixing.}
There exists a symmetric, doubly-stochastic matrix $W$ respecting $G$ with spectral gap $1-\lambda>0$ such that

\begin{equation}
\|E_{t+1}\|*{F} = \|X_t W - \mathbf{1}\bar x_t^\top\|*{F} \le \lambda\,\|E_t\|_{F}.
\end{equation}

\textbf{Asynchronous (gossip) mixing.}
Let $\{W_t^{(j)}\}$ be the sequence of row-stochastic mixing matrices arising from pairwise/asynchronous gossip.

Assume edges are activated with bounded delays and there exists a window size $B$ and a contraction factor $\tilde\lambda \in (0,1)$ such that
\begin{equation}
\mathbb{E}\Big[\|E_{t+B}\|*{F}^{2} \,\big|\, E_t \Big] \le \tilde\lambda\, \|E_t\|*{F}^{2}.
\end{equation}

These are standard consensus contraction assumptions in decentralized optimization
(i.e. \citep{lian2017, koloskova20a}).

\begin{lemma}[Assignment is descent for $F_{\text{global}}$]\label{lem:assign}
Conditioned on parameters $\{\theta^t_{i,j}\}$, the assignment step does not increase $F_{\text{global}}$:
\begin{equation} 
\sum_{i=1}^N \min_{j} F_{\text{client}}(\theta^t_{i,j},D_i)
\ \le\
\sum_{i=1}^N F_{\text{client}}(\theta^t_{i,c_{t-1}(i)},D_i).
 \end{equation}
\end{lemma}
\begin{proof}
Pointwise argmin over $j$ per client $i$ can only reduce the sum; cf.\ \cite{IFCA}.
\end{proof}

\begin{lemma}[Local SGD descent with disagreement penalty]\label{lem:localsgd}
Let $\gamma\le 1/L$. Then, conditioned on $\Theta^t$,
\begin{equation} 
\begin{split}
\mathbb E\!\left[F_{\text{cluster}}(j;\bar \theta^{t+\frac12}_j)\mid \Theta^t\right]\ \le\  \\
F_{\text{cluster}}(j;\bar \theta^{t}_j) 
-\frac{\gamma}{2}\|\nabla F_{\text{cluster}}(j;\bar \theta^t_j)\|^2
+\gamma^2 L\Big(\sigma^2 + L^2\,\mathrm{Disp}^t_j\Big).
\end{split}
\end{equation}
\end{lemma}
\begin{proof}
Apply the smoothness descent lemma to the cluster-sum objective using unbiased gradients, and decompose the error into stochastic noise $\sigma^2$ and a consensus term proportional to $\mathrm{Disp}^t_j$. This form follows standard decentralized SGD analyses, e.g. \cite{lian2017,stich2019,koloskova20a}.
\end{proof}

\begin{lemma}[Gossip preserves averages and contracts disagreement]\label{lem:gossip}
For each $j$,
\(
\bar \theta^{t+1}_j=\bar \theta^{t+\frac12}_j.
\)
Moreover, in the synchronous (fixed $W$) case,
\begin{equation} 
\mathbb E\!\left[\mathrm{Disp}^{t+1}_j\mid \Theta^{t+\frac12}\right]
\ \le\ \lambda^2\,\mathrm{Disp}^{t+\frac12}_j.
 \end{equation}
In the asynchronous case, for some window $B$ and $\tilde \lambda\in(0,1)$,
\(
\mathbb E[\mathrm{Disp}^{t+B}_j]\le \tilde \lambda^2\,\mathrm{Disp}^{t}_j.
\)
\end{lemma}
\begin{proof}
Average preservation follows from row-stochasticity (and doubly-stochasticity in the synchronous case). Disagreement evolution is governed by multiplication with $W^{(j)}_t$; contraction follows from the spectral gap (synchronous) or joint-connectivity arguments for randomized gossip \citep{boyd2006randomized,nedic2018network}.
\end{proof}

\begin{lemma}[Assignment stabilization]\label{lem:stabilize}
Under (A1)--(A5) with sufficiently small $\gamma$, there exists a finite $\tau$ such that $c_t(i)=c_\star(i)$ for all $i$ and all $t\ge \tau$.
\end{lemma}
\begin{proof}[Proof sketch]
By Lemmas~\ref{lem:localsgd}--\ref{lem:gossip}, the averages $\{\bar \theta^t_j\}$ descend and the disagreements $\mathrm{Disp}^t_j$ contract, so all client copies tracking a fixed $j$ enter and remain in a neighborhood of $\theta^\star_j$. Within this neighborhood, separability (A5) enforces a unique, correct argmin, hence stable assignments; cf. \cite{IFCA}.
\end{proof}
\begin{theorem}[Convergence of DFCA from Theorem\ref{thm:dfca}]
Assume (A1)--(A5), choose $\gamma\le c/L$ for a small numerical constant $c$, and let $\lambda$ (resp.\ $\tilde \lambda$) be the consensus factor in the synchronous (resp.\ async) case. Then:
\begin{enumerate}[label=(\roman*)]
\item (\emph{Pre-stabilization}) $F_{\text{global}}^t$ is non-increasing in expectation across assignment and local steps. The disagreements $\{\mathrm{Disp}^t_j\}$ remain bounded and contract at rate $\lambda$ (or $\tilde \lambda$ over windows).
\item (\emph{Stabilization}) There exists $\tau<\infty$ such that $c_t(i)=c_\star(i)$ for all $t\ge \tau$.
\item (\emph{Post-stabilization}) For $t\ge \tau$, DFCA is $k$ independent copies of decentralized SGD on $F_{\text{cluster}}(j)$.
\begin{itemize}
\item Under (PL) for some $\mu>0$,
\begin{equation} 
\begin{split}
\mathbb E\!\left[F_{\text{global}}^{\tau+T}-F_{\text{global}}^\star\right]
\ \le\ \\
(1-\mu\gamma/2)^T\,C_0
\ +\ O\!\Big(\tfrac{\gamma\sigma^2}{\mu}\Big)
\ +\ O\!\Big(\tfrac{\gamma L}{1-\lambda}\,\sigma^2\Big),
\end{split}
\end{equation}
with $C_0$ depending on the gap at $t=\tau$; in async, replace $(1-\lambda)$ by the windowed $(1-\tilde \lambda)$.
\item Under (Cvx),
\begin{equation} 
\begin{split}
\frac1T \sum_{t=\tau}^{\tau+T-1}\sum_{j=1}^k \mathbb E\|\nabla F_{\text{cluster}}(j;\bar \theta^t_j)\|^2
\ \le\ \\
O\!\Big(\frac{F_{\text{global}}^\tau-F_{\text{global}}^\star}{\gamma T}\Big)
+O(\gamma L\sigma^2)
+O\!\Big(\tfrac{\gamma L}{1-\lambda}\sigma^2\Big),
\end{split}
\end{equation}
and choosing $\gamma=\Theta(1/\sqrt{T})$ yields the usual $O(1/\sqrt{T})$ rates (with the consensus penalty).
\end{itemize}
\end{enumerate}
\end{theorem}
\begin{proof}[Proof sketch]
Combine Lemma~\ref{lem:assign} (assignment descent), Lemma~\ref{lem:localsgd} (SGD descent with a disagreement term), and Lemma~\ref{lem:gossip} (average preservation and disagreement contraction). Lemma~\ref{lem:stabilize} yields finite-time stabilization, after which each cluster index $j$ follows a standard decentralized SGD recursion; apply known rates under PL or convexity and sum over $j$.
\end{proof}

\paragraph{Remarks}
\emph{(i) Batch vs.\ sequential aggregation.} The sequential ``running average'' update
\(
\theta \leftarrow \frac{r}{r+1}\theta + \frac{1}{r+1}\theta_{\text{new}}
\)
implements a valid stochastic gossip step; the windowed contraction in Lemma~\ref{lem:gossip} covers it.
\emph{(ii) Initialization.} Global initialization (DFCA-GI) sets $\mathrm{Disp}^0_j=0$ and typically reduces $\tau$; local initialization (DFCA-LI) only changes constants.
\emph{(iii) Clients not training $j$.} They still mix $\theta_{i,j}$ by applying $W^{(j)}_t$ to their current value; average preservation and contraction remain valid.

\end{document}